\documentclass[12pt]{article}
\usepackage[centertags]{amsmath}
\usepackage{amsfonts}
\usepackage{amssymb}
\usepackage{latexsym}
\usepackage{amsthm}
\usepackage{newlfont}
\usepackage{graphicx}
\usepackage{listings}
\usepackage{booktabs}
\usepackage{abstract}
\newlength{\defbaselineskip}
\newcommand{\setlinespacing}[1]%
           {\setlength{\baselineskip}{#1 \defbaselineskip}}

\newcommand{\mspan}{\operatorname{span}}

\newcommand{\actaqed}{\hfill $\actabox$}
{\medskip\noindent \textit{Proof of #1. }}%
{\actaqed \medskip}

\def\D{{\mathcal D}}

\def\R{{\mathbb R}}
\def \<{\langle}
\def\>{\rangle}

\def \e{\epsilon}

\def \ff{\varphi}

\def \sp{\operatorname{span}}

\def\bt{\beta}
\def\la{\lambda}
\def\a{\alpha}

\newtheorem{Theorem}{Theorem}[section]
\newtheorem{Lemma}{Lemma}[section]
\newtheorem{Definition}{Definition}[section]
\newtheorem{Proposition}{Proposition}[section]
\newtheorem{Remark}{Remark}[section]

\newtheorem{Corollary}{Corollary}[section]
\numberwithin{equation}{section}

\begin{document}
\title{{Greedy approximation in convex optimization} }
\author{V.N. Temlyakov \thanks{ University of South Carolina. Research was supported by NSF grant DMS-0906260 }} \maketitle
\begin{abstract}
{We study sparse approximate solutions to convex optimization problems. It is known that in many engineering applications researchers are interested in 
an approximate solution of an optimization problem as a linear combination of elements from a given system of elements. There is an increasing interest in building such sparse approximate solutions using different greedy-type algorithms. The problem of approximation of a given element of a Banach space by linear combinations of elements from a given system (dictionary) is well studied 
in nonlinear approximation theory.  At a first glance the settings of 
approximation and optimization problems are very different. In the approximation problem an element is given and our task is to find a sparse approximation of it. In optimization theory an energy function is given and we should find an approximate sparse solution to the minimization problem. It turns out that the same technique can be used for solving both problems.
We show how the technique developed in nonlinear approximation theory,
in particular, the greedy approximation technique can be adjusted for finding a sparse solution of an optimization problem.    }
\end{abstract}

\section{Introduction}

We study sparse approximate solutions to convex optimization problems. We apply the technique developed in nonlinear approximation known under the name of {\it greedy approximation}. A typical 
problem of convex optimization is to find an approximate solution to the problem
\begin{equation}\label{1.0}
\inf_x E(x)
\end{equation}
under assumption that $E$ is a convex function. Usually, in convex optimization function $E$ is defined on a finite dimensional space $\R^n$ (see \cite{BL}, \cite{N}).
Recent needs of numerical analysis call for consideration of the above optimization problem on an infinite dimensional space, for instance, a space of 
continuous functions. One more important argument that motivates us to study 
this problem in the infinite dimensional space setting is the following. In many contemporary numerical applications the dimension $n$ of the ambient space $\R^n$ is large and we would like to obtain bounds on the convergence rate independent of the dimension $n$. Our results for infinite dimensional spaces 
provide such bounds on the convergence rate. 
Thus, we consider a convex function $E$ defined on a Banach space $X$. It is pointed out in \cite{Z} that in many engineering applications researchers are interested in 
an approximate solution of problem (\ref{1.0}) as a linear combination of elements from a given system $\D$ of elements. There is an increasing interest in building such sparse approximate solutions using different greedy-type algorithms (see, for instance, \cite{Z}, \cite{SSZ}, \cite{CRPW}, and \cite{TRD}). The problem of approximation of a given element $f\in X$ by linear combinations of elements from $\D$ is well studied 
in nonlinear approximation theory (see, for instance \cite{D}, \cite{T2}, \cite{Tbook}). In order to address the contemporary needs of approximation theory and computational mathematics, a very general model of approximation with regard to a redundant system (dictionary) has been considered in many recent papers.   As such a model, we choose a Banach space $X$ with elements as target functions and an arbitrary system $\D$ of elements of this space such that the closure of $\mspan\D$ coincides with $X$ as an approximating system. 
 
The fundamental question  is how to construct good methods (algorithms) of approximation. Recent results have established   that greedy type algorithms are suitable methods of nonlinear approximation in both sparse approximation with regard to bases and sparse approximation with regard to redundant systems. 
It turns out that there is one fundamental principal that allows us to build good algorithms both for arbitrary redundant systems and for very simple well structured bases like the Haar basis. This principal is the use of a greedy step in searching for a new element to be added to a given sparse approximant. By a {\it greedy step}, we mean one 
which maximizes a certain functional determined by information from the previous steps of the algorithm. We obtain different types of greedy algorithms by varying the above mentioned functional and also by using different ways of 
constructing (choosing coefficients of the linear combination) the $m$-term approximant from the already found $m$ elements of the dictionary.

  We point out that at a first glance the settings of 
approximation and optimization problems are very different. In the approximation problem an element $f\in X$ is given and our task is to find a sparse approximation of it. In optimization theory an energy function $E(x)$ is given and we should find an approximate sparse solution to the minimization problem. It turns out that the same technique can be used for solving both problems.

We show how the technique developed in nonlinear approximation theory,
in particular, the greedy approximation technique can be adjusted for finding a sparse with respect to $\D$ solution of problem (\ref{1.0}).  

We begin with a brief description of greedy approximation methods in Banach spaces. The reader can find a detailed discussion of greedy approximation in the book \cite{Tbook}. 
Let $X$ be a Banach space with norm $\|\cdot\|$. We say that a set of elements (functions) $\D$ from $X$ is a dictionary, respectively, symmetric dictionary, if each $g\in \D$ has norm bounded by one ($\|g\|\le1$),
$$
g\in \D \quad \text{implies} \quad -g \in \D,
$$
and the closure of $\sp \D$ is $X$. In this paper symmetric dictionaries are considered. We denote the closure (in $X$) of the convex hull of $\D$ by $A_1(\D)$.  For a nonzero element $f\in X$ we let $F_f$ denote a norming (peak) functional for $f$: 
$$
\|F_f\| =1,\qquad F_f(f) =\|f\|.
$$
The existence of such a functional is guaranteed by Hahn-Banach theorem. 
We describe a typical greedy algorithm from a family of {\it dual greedy algorithms}. 
Let 
$\tau := \{t_k\}_{k=1}^\infty$ be a given weakness sequence of  nonnegative numbers $t_k \le 1$, $k=1,\dots$. We define first the Weak Chebyshev Greedy Algorithm (WCGA) (see \cite{T15}) that is a generalization for Banach spaces of the Weak Orthogonal Greedy Algorithm.   

 {\bf Weak Chebyshev Greedy Algorithm (WCGA).} 
We define $f^c_0 := f^{c,\tau}_0 :=f$. Then for each $m\ge 1$ we have the following inductive definition.

(1) $\varphi^{c}_m :=\varphi^{c,\tau}_m \in \D$ is any element satisfying
$$
F_{f^{c}_{m-1}}(\varphi^{c}_m) \ge t_m  \sup_{g\in\D}F_{f^{c}_{m-1}}(g).
$$

(2) Define
$$
\Phi_m := \Phi^\tau_m := \sp \{\varphi^{c}_j\}_{j=1}^m,
$$
and define $G_m^c := G_m^{c,\tau}$ to be the best approximant to $f$ from $\Phi_m$.

(3) Let
$$
f^{c}_m := f^{c,\tau}_m := f-G^c_m.
$$

Let us make a remark that justifies the idea of the dual greedy algorithms in terms of real analysis.  We consider here approximation in uniformly smooth Banach spaces. For a Banach space $X$ we define the modulus of smoothness
$$
\rho(u) := \sup_{\|x\|=\|y\|=1}(\frac{1}{2}(\|x+uy\|+\|x-uy\|)-1).
$$
The uniformly smooth Banach space is the one with the property
$$
\lim_{u\to 0}\rho(u)/u =0.
$$
 
We note that from the definition of modulus of smoothness we get the following inequality.
\begin{equation}\label{1.1'}
0\le \|x+uy\|-\|x\|-uF_x(y)\le 2\|x\|\rho(u\|y\|/\|x\|).  
\end{equation}
This inequality implies the proposition.
\begin{Proposition}\label{P1.1} Let $X$ be a uniformly smooth Banach space. Then, for any $x\neq0$ and $y$ we have
\begin{equation}\label{1.9}
F_x(y)=\left(\frac{d}{du}\|x+uy\|\right)(0)=\lim_{u\to0}(\|x+uy\|-\|x\|)/u. 
\end{equation}
\end{Proposition}
Proposition \ref{P1.1} shows that in the WCGA we are looking for an element $\ff_m\in\D$ that provides a big derivative of the quantity $\|f_{m-1}+ug\|$. 
Here is one more important greedy algorithm. 

  {\bf Weak Greedy Algorithm with Free Relaxation  (WGAFR).} 
Let $\tau:=\{t_m\}_{m=1}^\infty$, $t_m\in[0,1]$, be a weakness  sequence. We define $f_0   :=f$ and $G_0  := 0$. Then for each $m\ge 1$ we have the following inductive definition.

(1) $\varphi_m   \in \D$ is any element satisfying
$$
F_{f_{m-1}}(\varphi_m  ) \ge t_m   \sup_{g\in\D}F_{f_{m-1}}(g).
$$

(2) Find $w_m$ and $ \lambda_m$ such that
$$
\|f-((1-w_m)G_{m-1} + \la_m\varphi_m)\| = \inf_{ \la,w}\|f-((1-w)G_{m-1} + \la\varphi_m)\|
$$
and define
$$
G_m:=   (1-w_m)G_{m-1} + \la_m\varphi_m.
$$

(3) Let
$$
f_m   := f-G_m.
$$
It is known that both algorithms WCGA and WGAFR converge in any uniformly smooth Banach space under mild conditions on the weakness sequence $\{t_k\}$, for instance, $t_k=t$, $k=1,2,\dots$, $t>0$, guarantees such convergence. The following theorem provides rate of convergence (see \cite{Tbook}, pp. 347, 353). 
\begin{Theorem}\label{T1.1} Let $X$ be a uniformly smooth Banach space with modulus of smoothness $\rho(u)\le \gamma u^q$, $1<q\le 2$. Take a number $\e\ge 0$ and two elements $f$, $f^\e$ from $X$ such that
$$
\|f-f^\e\| \le \e,\quad
f^\e/A(\e) \in A_1(\D),
$$
with some number $A(\e)>0$.
Then, for both algorithms WCGA and WGAFR  we have ($p:=q/(q-1)$)
$$
\|f^{c,\tau}_m\| \le  \max\left(2\e, C(q,\gamma)(A(\e)+\e)(1+\sum_{k=1}^mt_k^p)^{-1/p}\right) . 
$$
\end{Theorem}

The above Theorem \ref{T1.1} simultaneously takes care of two issues: noisy data and approximation in an interpolation space. In order to apply it for noisy data we interpret $f$ as a noisy version of a signal and $f^\e$ as a noisless version of a signal. Then, assumption $f^\e/A(\e)\in A_1(\D)$ describes our smoothness assumption on the noisless signal. Theorem \ref{T1.1} can be applied for approximation of $f$ under assumption that $f$ belongs  to one  
of interpolation spaces between $X$ and the space generated by the $A_1(\D)$-norm (atomic norm). We now make a remark showing that the $A_1(\D)$-norm (in other words, the assumption $f/A\in A_1(\D)$) appears naturally in convex optimization problems. 

It is pointed out in \cite{FNW} that there has been considerable interest in solving the convex unconstrained optimization problem 
\begin{equation}\label{op1}
\min_{x}\frac{1}{2}\|y-\Phi x\|_2^2 +\la\|x\|_1
\end{equation}
where $x\in \R^n$, $y\in \R^k$, $\Phi$ is an $k\times n$ matrix, $\la$ is a nonnegative parameter, $\|v\|_2$ denotes the Euclidian norm of $v$, and 
$\|v\|_1$ is the $\ell_1$ norm of $v$. Problems of the form (\ref{op1}) have become familiar over the past three decades, particularly in statistical and signal processing contexts. Problem (\ref{op1}) is closely related to the following convex constrained optimization problem
\begin{equation}\label{op4}
\min_{x}\frac{1}{2}\|y-\Phi x\|_2^2 \quad\text{subject to}\quad \|x\|_1\le A.
\end{equation}
The above convex optimization problem can be recast as an approximation problem of $y$ with respect to a dictionary $\D:=\{\pm\varphi_i\}_{i=1}^n$ which is associated with a $k\times n$ matrix
$\Phi=[\varphi_1\dots\varphi_n]$ with $\varphi_j\in \R^k$ being the column vectors of $\Phi$. The condition $y\in A_1(\D)$ is
equivalent to existence of $x\in \R^m$ such that $y=\Phi x$ and
\begin{equation}\label{op5}
\|x\|_{1}:=|x_1|+\dots+|x_m| \le 1.
\end{equation}
As a direct corollary of Theorem \ref{T1.1}, we get for any
$y\in A_1(\D)$ that the WCGA and the WGAFR with $\tau=\{t\}$ guarantee the
following upper bound for the error
\begin{equation}\label{op6}
\|y_k\|_2\le Ck^{-1/2}.
\end{equation}
The bound (\ref{op6}) holds for any $\D$ (any $\Phi$).

We note that in the study of greedy-type algorithms in approximation theory (see \cite{Tbook}) emphasis are put on the theory of approximation with respect to arbitrary dictionary $\D$. The reader can find examples of specific dictionaries of interest in \cite{Tbook} and \cite{TRD}. We present some results 
on sparse solutions for convex optimization problems in the setting with an arbitrary dictionary $\D$.

We generalize the algorithms WCGA and WGAFR to the case of convex optimization and prove an analog of Theorem \ref{T1.1} for the new algorithms. Let us illustrate this on the generalization of the WGAFR.

We 
assume that the set
$$
D:=\{x:E(x)\le E(0)\}
$$
is bounded.
For a bounded set $D$ define the modulus of smoothness of $E$ on $D$ as follows
\begin{equation}\label{1.1}
\rho(E,u):=\frac{1}{2}\sup_{x\in D, \|y\|=1}|E(x+uy)+E(x-uy)-2E(x)|.
\end{equation}

We assume that $E$ is Fr{\'e}chet differentiable. Then convexity of $E$ implies that for any $x,y$ 
\begin{equation}\label{1.2}
E(y)\ge E(x)+\<E'(x),y-x\>
\end{equation}
or, in other words,
\begin{equation}\label{1.3}
E(x)-E(y) \le \<E'(x),x-y\> = \<-E'(x),y-x\>.
\end{equation} 
We will often use the following simple lemma.
\begin{Lemma}\label{L1.1} Let $E$ be Fr{\'e}chet differentiable convex function. Then the following inequality holds for $x\in D$
\begin{equation}\label{1.6}
0\le E(x+uy)-E(x)-u\<E'(x),y\>\le 2\rho(E,u\|y\|).  
\end{equation}
\end{Lemma}
\begin{proof} The left inequality follows directly from (\ref{1.2}).
  Next, from the definition of modulus of smoothness it follows that
\begin{equation}\label{1.7}
E(x+uy)+E(x-uy)\le 2(E(x)+\rho(E,u\|y\|)).  
\end{equation}
Inequality (\ref{1.2}) gives
\begin{equation}\label{1.8}
E(x-uy)\ge E(x) + \<E'(x),-uy\>=E(x)-u\<E'(x),y\>. 
\end{equation}
Combining (\ref{1.7}) and (\ref{1.8}), we obtain
$$
E(x+uy)\le E(x)+u\<E'(x),y\>+2\rho(E,u\|y\|).
$$
This proves the second inequality. 
\end{proof}

  {\bf Weak Greedy Algorithm with Free Relaxation  (WGAFR(co)).} 
Let $\tau:=\{t_m\}_{m=1}^\infty$, $t_m\in[0,1]$, be a weakness  sequence. We define   $G_0  := 0$. Then for each $m\ge 1$ we have the following inductive definition.

(1) $\varphi_m   \in \D$ is any element satisfying
$$
\<-E'(G_{m-1}),\varphi_m\> \ge t_m  \sup_{g\in\D}\<-E'(G_{m-1}),g\>.
$$

(2) Find $w_m$ and $ \lambda_m$ such that
$$
E((1-w_m)G_{m-1} + \la_m\varphi_m) = \inf_{ \la,w}E((1-w)G_{m-1} + \la\varphi_m)
$$
and define
$$
G_m:=   (1-w_m)G_{m-1} + \la_m\varphi_m.
$$
 
In Section 4 we prove the following rate of convergence result.

\begin{Theorem}\label{T1.2} Let $E$ be a uniformly smooth convex function with modulus of smoothness $\rho(E,u)\le \gamma u^q$, $1<q\le 2$. Take a number $\e\ge 0$ and an element  $f^\e$ from $D$ such that
$$
E(f^\e) \le \inf_{x\in D}E(x)+ \e,\quad
f^\e/A(\e) \in A_1(\D),
$$
with some number $A(\e)\ge 1$.
Then we have for WGAFR(co) ($p:=q/(q-1)$)
$$
E(G_m)-\inf_{x\in D}E(x) \le  \max\left(2\e, C(q,\gamma)A(\e)\left(C(E,q,\gamma)+\sum_{k=1}^mt_k^p\right)^{1-q}\right) . 
$$
\end{Theorem}

 We note that in all algorithms studied in this paper the sequence $\{G_m\}_{m=0}^\infty$ of approximants satisfies the conditions
 $$
 G_0=0,\quad E(G_0)\ge E(G_1) \ge E(G_2) \ge \dots .
 $$
 This guarantees that $G_m\in D$ for all $m$.
 
 This paper is the first author's paper on greedy-type methods in convex optimization. It is a slight modification of the paper \cite{Tco1}. For the reader's 
 convenience we now give a brief general description and classification of greedy-type algorithms for convex optimization. The most difficult part of an algorithm is to find an element $\ff_m\in\D$ to be used in approximation process. We consider greedy methods for finding $\ff_m\in\D$. We have two types of greedy steps to find $\ff_m\in\D$.
 
 {\bf I. Gradient greedy step.} At this step we look for an element $\ff_m\in\D$ such that
 $$
\<-E'(G_{m-1}),\varphi_m\> \ge t_m  \sup_{g\in\D}\<-E'(G_{m-1}),g\>.
$$

 {\bf II. $E$-greedy step.} At this step we look for an element $\ff_m\in\D$ which satisfies (we assume existence):
 $$
\inf_{c\in\R}E(G_{m-1}+c\varphi_m) = \inf_{g\in\D,c\in\R}E(G_{m-1}+cg) .
$$

The above WGAFR(co) uses the greedy step of type {\bf I}. In this paper we only discuss algorithms based on the greedy step of type {\bf I}. These algorithms fall into a category of the first order methods. The greedy step of type {\bf II} uses only the function values $E(x)$. We discussed some of the algorithms of this type in \cite{Tco2} and plan to study them in our future work. 

After we found $\ff_m\in\D$ we can proceed in different ways. We now list some typical steps that are motivated by the corresponding steps in greedy approximation theory (see \cite{Tbook}). These steps or their variants are used in optimization algorithms like {\it gradient method}, {\it reduced gradient method}, {\it conjugate gradients}, {\it gradient pursuits} (see, for instance, \cite{FW}, \cite{N}, \cite{K}, \cite{PD}, \cite{BD1} and \cite{BD2}).

(A) Best step in the direction $\ff_m\in\D$. We choose $c_m$ such that
$$
E(G_{m-1}+c_m\varphi_m) = \inf_{c\in\R}E(G_{m-1}+c\ff_m) 
$$
and define 
$$
G_m:=G_{m-1}+c_m\ff_m.
$$

(B) Reduced best step in the direction $\ff_m\in\D$. We choose $c_m$ as in (A) and for a given parameter $b>0$ define
$$
G_m^b:=G_{m-1}^b+bc_m\ff_m.
$$
Usually, $b\in(0,1)$. This is why we call it {\it reduced}.

(C) Chebyshev-type methods. We choose $G_m\in \sp(\ff_1,\dots,\ff_m)$ which satisfies
$$
E(G_m)=\inf_{c_j,j=1,\dots,m}E(c_1\ff_1+\cdots+c_m\ff_m).
$$

(D) Fixed relaxation. For a given sequence $\{r_k\}_{k=1}^\infty$ of relaxation parameters $r_k\in [0,1)$ we choose $G_m:=(1-r_m)G_{m-1}+c_m\ff_m$ with $c_m$ from
$$
E((1-r_m)G_{m-1}+c_m\ff_m)=\inf_{c\in\R}E((1-r_m)G_{m-1}+c\ff_m).
$$

(F) Free relaxation. We choose $G_m\in \sp(G_{m-1},\ff_m)$ which satisfies
$$
E(G_m)=\inf_{c_1,c_2}E(c_1G_{m-1}+ c_2\ff_m).
$$

(G) Prescribed coefficients. For a given sequence $\{c_k\}_{k=1}^\infty$ of positive coefficients in the case of greedy step {\bf I} we define 
\begin{equation}\label{1.14}
G_m:=G_{m-1}+c_m\ff_m.
\end{equation}
In the case of greedy step {\bf II} we define $G_m$ by formula (\ref{1.14}) with the greedy step {\bf II} modified as follows: $\ff_m\in\D$ is an element satisfying
$$
E(G_{m-1}+c_m\ff_m)=\inf_{g\in\D}E(G_{m-1}+c_mg).
$$

We prove convergence and rate of convergence results here. Our setting in an infinite dimensional Banach space makes the convergence results nontrivial. The rate of convergence results are of interest in both finite dimensional and infinite dimensional settings. In these results we make  assumptions on the element minimizing $E(x)$ (in other words we look for $\inf_{x\in S}E(x)$ for a special domain $S$). A typical assumption in this regard is formulated in terms of the convex hull $A_1(\D)$ of the dictionary $\D$. 

We have already mentioned above (see (\ref{op4}) and below) an example which is of interest in applications in compressed sensing. We now mention another example that attracted a lot of attention in the recent literature. In this example $X$ is a Hilbert space of all real matrices of size $n\times n$ equipped with the Frobenius norm $\|\cdot\|_F$. A dictionary $\D$ is the set of all matrices of rank one normalized in the Frobenius norm. In this case $A_1(\D)$ is the set of matrices with nuclear norm not exceeding $1$. We are interested in 
sparse minimization of $E(x):=\|f-x\|_F^2$ (sparse approximation of $f$) with respect to $\D$.

\section{ The Weak Chebyshev Greedy Algorithm} 

 We begin with the following two simple and well-known lemmas.
\begin{Lemma}\label{L2.1} Let $E$ be a uniformly smooth convex function on a Banach space $X$ and $L$ be a finite-dimensional subspace of $X$. Let $x_L$ denote the point from $L$ at which $E$ attains the minimum:
$$
E(x_L)=\inf_{x\in L}E(x).
$$ 
 Then we have 
$$
\<E'(x_L),\phi\> =0
$$
for any $\phi \in L$.
\end{Lemma}
\begin{proof} Let us assume the contrary: there is a $\phi \in L$ such that $\|\phi\|=1$ and
$$
\<E'(x_L),\phi\> =\bt >0.
$$
It is clear that $x_L\in L\cap D$. For any $\la$ we have from the definition of $\rho(E,\la)$ that 
\begin{equation}\label{2.3}
E(x_L-\la \phi) +E(x_L+\la \phi) \le 2(E(x_L)+\rho(E,\la)).
\end{equation}
Next by (\ref{1.2})
\begin{equation}\label{2.4}
 E(x_L+\la \phi) \ge E(x_L) +\< E'(x_L),\la\phi\> =E(x_L)+\la\bt.
\end{equation}
Combining (\ref{2.3}) and (\ref{2.4}) we get
\begin{equation}\label{2.5}
E(x_L-\la\phi)\le E(x_L)-\la\bt +2\rho(E,\la).
\end{equation}
Taking into account that $\rho(E,u) =o(u)$, we find $\la'>0$ such that
$$
-\la'\bt +2\rho(E,\la') <0.
$$
Then (\ref{2.5}) gives
$$ 
E(x_L-\la'\phi) < E(x_L),
$$
which contradicts the assumption that $x_L \in L$ is the point of minimum of $E$.
\end{proof}
\begin{Lemma}\label{L2.2} For any bounded linear functional $F$ and any dictionary $\D$, we have
$$
 \sup_{g\in \D}\<F,g\> = \sup_{f\in  A_1(\D)} \<F,f\>.
$$
\end{Lemma}
\begin{proof} The inequality 
$$
\sup_{g\in \D}\<F,g\> \le \sup_{f\in A_1(\D)} \<F,f\>
$$
is obvious. We prove the opposite inequality. Take any $f\in A_1(\D)$. Then for any $\e >0$ there exist $g_1^\e,\dots,g_N^\e \in \D$ and numbers $a_1^\e,\dots,a_N^\e$ such that $a_i^\e>0$, $a_1^\e+\dots+a_N^\e \le 1$ and 
$$
\|f-\sum_{i=1}^Na_i^\e g_i^\e\| \le \e.
$$
Thus
$$
\<F,f\> \le \|F\|\e + \<F,\sum_{i=1}^Na_i^\e g_i^\e\> \le \e \|F\| +\sup_{g\in \D} \<F,g\>
$$
which proves Lemma \ref{L2.2}.
\end{proof}

We define the following generalization of the WCGA for convex optimization.

 {\bf Weak Chebyshev Greedy Algorithm (WCGA(co)).} 
We define $G_0 := 0$. Then for each $m\ge 1$ we have the following inductive definition.

(1) $\varphi_m :=\varphi^{c,\tau}_m \in \D$ is any element satisfying
$$
\<-E'(G_{m-1}),\varphi_m\> \ge t_m  \sup_{g\in \D}\< -E'(G_{m-1}),g\>.
$$

(2) Define
$$
\Phi_m := \Phi^\tau_m := \sp \{\varphi_j\}_{j=1}^m,
$$
and define $G_m := G_m^{c,\tau}$ to be the point from $\Phi_m$ at which $E$ attains the minimum:
$$
E(G_m)=\inf_{x\in \Phi_m}E(x).
$$ 

The following lemma is a key lemma in studying convergence and rate of convergence of WCGA(co).
\begin{Lemma}\label{L2.3} Let $E$ be a uniformly smooth convex function   with modulus of smoothness $\rho(E,u)$. Take a number $\e\ge 0$ and an element  $f^\e$ from $D$ such that
$$
E(f^\e) \le \inf_{x\in X}E(x)+ \e,\quad
 f^\e/A(\e) \in A_1(\D),
$$
with some number $A(\e)\ge 1$.
Then we have for the WCGA(co)
$$
E(G_m)-E(f^\e) \le E(G_{m-1})-E(f^\e) 
$$
$$
+ \inf_{\la\ge0}(-\la t_mA(\e)^{-1}(E(G_{m-1})-E(f^\e)) + 2\rho(E,\la)),
$$
for $ m=1,2,\dots$ .
\end{Lemma}
\begin{proof} It follows from the definition of WCGA(co) that $E(0)\ge E(G_1)\ge E(G_2)\dots$. Therefore, if  $E(G_{m-1})-E(f^\e)\le 0$ then the claim of Lemma \ref{L2.3} is trivial. Assume $E(G_{m-1})-E(f^\e)>0$. By Lemma \ref{L1.1} we have for any $\la$
\begin{equation}\label{2.6}
E(G_{m-1}+\la \varphi_m) \le E(G_{m-1}) - \la\<-E'(G_{m-1}),\varphi_m\> + 2 \rho(E,\la)
\end{equation}
and by (1) from the definition of the WCGA(co) and Lemma \ref{L2.2} we get
$$
\<-E'(G_{m-1}),\varphi_m\> \ge t_m \sup_{g\in \D} \<-E'(G_{m-1}),g\> =
$$
$$
t_m\sup_{\phi\in A_1(\D)} \<-E'(G_{m-1}),\phi\> \ge t_m A(\e)^{-1} \<-E'(G_{m-1}),f^\e\>.
$$
By Lemma \ref{L2.1} and (\ref{1.3}) we obtain
$$
\<-E'(G_{m-1}),f^\e\> = \<-E'(G_{m-1}),f^\e-G_{m-1}\> \ge E(G_{m-1})-E(f^\e).
$$
 Thus,  
$$
E(G_m) \le \inf_{\la\ge0} E(G_{m-1}+ \la\varphi_m)  
$$
\begin{equation}\label{2.7}
\le E(G_{m-1}) + \inf_{\la\ge0}(-\la t_mA(\e)^{-1}(E(G_{m-1})-E(f^\e)) + 2\rho(E,\la),  
\end{equation}
which proves the lemma.
\end{proof}

We  proceed to a theorem on convergence of the WCGA. In the formulation of this theorem we need a special sequence which is defined for a given modulus of smoothness $\rho(u)$ and a given $\tau = \{t_k\}_{k=1}^\infty$.
\begin{Definition}\label{D2.1} Let $\rho(E,u)$ be an even convex function on $(-\infty,\infty)$ with the property:  
$$
\lim_{u\to 0}\rho(E,u)/u =0.
$$
For any $\tau = \{t_k\}_{k=1}^\infty$, $0<t_k\le 1$, and $\theta>0$
 we define $\xi_m := \xi_m(\rho,\tau,\theta)$ as a number $u$ satisfying the equation
\begin{equation}\label{2.1}
\rho(E,u) = \theta t_m u.  
\end{equation}
\end{Definition}
\begin{Remark}\label{R2.2} Assumptions on $\rho(E,u)$ imply that the function
$$
s(u) := \rho(E,u)/u, \quad u\neq 0,\quad s(0) =0,
$$
is a continuous increasing function on $[0,\infty)$. Thus \ref{2.1} has a unique solution $\xi_m=s^{-1}(\theta t_m)$ such that $\xi_m>0$ for $\theta\le \theta_0:=s(2)$. In this case we have $ \xi_m(\rho,\tau,\theta)\le 2$.
\end{Remark}

\begin{Theorem}\label{T2.1o} Let $E$ be a uniformly smooth convex function with modulus of smoothness $\rho(E,u)$. Assume that a sequence $\tau :=\{t_k\}_{k=1}^\infty$ satisfies the condition: for any $\theta \in(0,\theta_0]$ we have
$$
\sum_{m=1}^\infty t_m \xi_m(\rho,\tau,\theta) =\infty.
$$
 Then  
$$
\lim_{m\to \infty} E(G_m) =\inf_{x\in D}E(x).
$$
\end{Theorem}
\begin{Corollary}\label{2.1o} Let a convex function $E$ have modulus of smoothness $\rho(E,u)$ of power type $1<q\le 2$, that is, $\rho(E,u) \le \gamma u^q$. Assume that 
\begin{equation}\label{2.2o}
\sum_{m=1}^\infty t_m^p =\infty, \quad p=\frac{q}{q-1}.  
\end{equation}
Then  
$$
\lim_{m\to \infty} E(G_m) =\inf_{x\in D}E(x).
$$
\end{Corollary}
\begin{proof} The definition of the WCGA(co) implies that $\{E(G_m)\}$ is a non-increasing sequence. Therefore we have
$$
\lim_{m\to \infty}E(G_m) =a.
$$
Denote
$$
b:=\inf_{x\in D}E(x),\quad \a:= a-b.
$$
We prove that $\a =0$ by contradiction. Assume to the contrary that $\a>0$. Then, for any $m$ we have
$$
E(G_m)-b \ge \a.
$$
We set $\e =\a/2$ and find $f^\e$ such that
$$
E(f^\e) \le b+\e \quad \text{and}\quad f^\e/A(\e) \in A_1(\D),
$$
with some $A(\e)\ge 1$. Then, by Lemma \ref{L2.3} we get
$$
E(G_m)-E(f^\e)  \le E(G_{m-1}) -E(f^\e) + \inf_{\la\ge0} (-\la t_mA(\e)^{-1}\a/2    +2\rho(E,\la)).
$$
 
Let us specify $\theta:=\min\left(\theta_0,\frac{\a}{8A(\e)}\right)$ and take $\la =  \xi_m(\rho,\tau,\theta)$. Then we obtain
$$
E(G_m) \le E(G_{m-1}) -2\theta t_m\xi_m.
$$
The assumption
$$
\sum_{m=1}^\infty t_m\xi_m =\infty
$$
brings a contradiction, which proves the theorem.
\end{proof} 
\begin{Theorem}\label{T2.3o} Let $E$ be a uniformly smooth convex function with modulus of smoothness $\rho(E,u)\le \gamma u^q$, $1<q\le 2$. Take a number $\e\ge 0$ and an element  $f^\e$ from $D$ such that
$$
E(f^\e) \le \inf_{x\in D}E(x)+ \e,\quad
f^\e/A(\e) \in A_1(\D),
$$
with some number $A(\e)\ge 1$.
Then we have for the WCGA(co) ($p:=q/(q-1)$)
\begin{equation}\label{2.11o}
E(G_m)-\inf_{x\in D}E(x) \le  \max\left(2\e, C(q,\gamma)A(\e)^q\left(C(E,q,\gamma)+\sum_{k=1}^mt_k^p\right)^{1-q}\right) . 
\end{equation}
\end{Theorem}
\begin{proof} Denote
$$
a_n:=E(G_n)-E(f^\e).
$$
The sequence $\{a_n\}$ is non-increasing. If $a_n\le0$ for some $n\le m$ then 
$E(G_m)-E(f^\e)\le 0$ and $E(G_m)-\inf_{x\in D}E(x) \le \e$ which implies (\ref{2.11o}). Thus we assume that $a_n>0$ for $n\le m$.

By Lemma \ref{L2.3}   we have  
\begin{equation}\label{2.8o}
a_m \le a_{m-1}+\inf_{\la\ge0} \left(-\frac{\la t_m a_{m-1}}{A(\e)} + 2\gamma \la^q\right).
\end{equation}
Choose $\la$ from the equation
$$
\frac{\la t_m a_{m-1}}{A(\e)} = 4\gamma \la^q
$$
which implies that
$$
\la =\left(\frac{ t_m a_{m-1}}{4\gamma A(\e)}\right)^{\frac{1}{q-1}} .
$$
Let 
$$
A_q := 2(4\gamma)^{\frac{1}{q-1}}.
$$
Using the notation $p:= \frac{q}{q-1}$ we get from (\ref{2.8o})
$$
a_m \le a_{m-1}\left(1-\frac{\la t_m}{2A(\e)} \right) = a_{m-1}(1-t_m^pa_{m-1}^{\frac{1}{q-1}}/(A_qA(\e)^{p})).
$$
Raising both sides of this inequality to the power $\frac{1}{q-1}$ and taking into account the inequality $x^r\le x$ for $r\ge 1$, $0\le x\le 1$, we obtain
$$
a_m^{\frac{1}{q-1}} \le a_{m-1}^{\frac{1}{q-1}} (1-t^p_ma_{m-1}^{\frac{1}{q-1}}/(A_qA(\e)^{p})).
$$
We now need a simple known lemma (see \cite{T13}).
\begin{Lemma}\label{L2.4} Suppose that a sequence $y_1\ge y_2\ge \dots \ge 0$ satisfies inequalities
$$
y_k\le y_{k-1}(1-w_ky_{k-1}),\quad w_k\ge 0,
$$
for $k>n$. Then for $m>n$ we have
$$
\frac{1}{y_m}\ge \frac{1}{y_n}+\sum_{k=n+1}^m w_k.
$$
\end{Lemma}
\begin{proof} It follows from the chain of inequalities
$$
\frac{1}{y_k}\ge \frac{1}{y_{k-1}}(1-w_ky_{k-1})^{-1} \ge \frac{1}{y_{k-1}}(1+w_ky_{k-1})=\frac{1}{y_{k-1}}+w_k.
$$
\end{proof}
By   Lemma \ref{L2.4} with $y_k:=a_k^{\frac{1}{q-1}}$, $n=0$, $w_k=t^p_m/(A_qA(\e)^{p})$  we get
$$
a_m^{\frac{1}{q-1}} \le C_1(q,\gamma) A(\e)^{p}\left(C(E,q,\gamma)+\sum_{n=1}^m t_n^p\right)^{-1}
$$
which implies
$$
a_m\le  C(q,\gamma)A(\e)^q\left(C(E,q,\gamma)+\sum_{n=1}^m t_n^p\right)^{1-q}.
$$
Theorem \ref{T2.3o} is now proved.
\end{proof}

 \section{Relaxation. Co-convex approximation} 

In this section we study a generalization for optimization problem of  relaxed greedy algorithms in Banach spaces considered in \cite{T15}. Let 
$\tau := \{t_k\}_{k=1}^\infty$ be a given weakness sequence of   numbers $t_k \in[0,1]$, $k=1,\dots$. 

 {\bf Weak Relaxed Greedy Algorithm (WRGA(co)).} 
We define   $G_0:=G^{r,\tau}_0 := 0$. Then, for each $m\ge 1$ we have the following inductive definition.

(1) $\varphi_m := \varphi^{r,\tau}_m \in \D$ is any element satisfying
$$
\<-E'(G_{m-1}),\varphi_m - G_{m-1}\> \ge t_m \sup_{g\in \D} \<-E'(G_{m-1}),g - G_{m-1}\>.
$$

(2) Find $0\le \lambda_m \le 1$ such that
$$
E((1-\la_m)G_{m-1} + \la_m\varphi_m) = \inf_{0\le \la\le 1}E((1-\la)G_{m-1} + \la\varphi_m)
$$
and define
$$
G_m:= G^{r,\tau}_m := (1-\la_m)G_{m-1} + \la_m\varphi_m.
$$

\begin{Remark}\label{R3.1} It follows from the definition of  the WRGA that the sequence   $\{E(G_m)\}$ is a non-increasing sequence.
\end{Remark}
We call the WRGA(co) {\it relaxed} because at the $m$th step of the algorithm we use a linear combination (convex combination) of the previous approximant $G_{m-1}$ and a new element $\varphi_m$. The relaxation parameter $\lambda_m$ in the WRGA(co) is chosen at the $m$th step depending on $E$. 
We prove here the analogs of Theorems \ref{T2.1o} and \ref{T2.3o} for the Weak Relaxed Greedy Algorithm.
\begin{Theorem}\label{T3.1} Let $E$ be a uniformly smooth convex function with modulus of smoothness $\rho(E,u)$. Assume that a sequence $\tau :=\{t_k\}_{k=1}^\infty$ satisfies the condition: for any $\theta \in(0,\theta_0]$ we have
$$
\sum_{m=1}^\infty t_m \xi_m(\rho,\tau,\theta) =\infty.
$$
 Then, for the WRGA(co) we have 
$$
\lim_{m\to \infty} E(G_m) =\inf_{x\in A_1(\D)}E(x).
$$
\end{Theorem}

\begin{Theorem}\label{T3.2} Let $E$ be a uniformly smooth convex function with modulus of smoothness $\rho(E,u) \le \gamma u^q$, $1<q\le 2$. Then, for a sequence $\tau := \{t_k\}_{k=1}^\infty$, $t_k \le 1$, $k=1,2,\dots,$ we have for any $f\in A_1(\D)$ that 
$$
E(G_m)-E(f) \le  \left(1+C_1(q,\gamma)\sum_{k=1}^m t_k^p\right)^{1-q},\quad p:= \frac{q}{q-1},
$$
with a positive constant $C_1(q,\gamma)$ which may depend only on $q$ and $\gamma$.
\end{Theorem}
\begin{proof}  This proof is similar to the proof of Theorems \ref{T2.1o} and \ref{T2.3o}. Instead of Lemma \ref{L2.3} we use the following lemma.
\begin{Lemma}\label{L3.1} Let $E$ be a uniformly smooth convex function with modulus of smoothness $\rho(E,u)$. Then, for any $f\in A_1(\D)$ we have
$$
E(G_m) \le E(G_{m-1} )+ \inf_{0\le\la\le 1}(-\la t_m (E(G_{m-1} )-E(f))+ 2\rho(E, 2\la)),
\quad m=1,2,\dots .
$$
\end{Lemma}
\begin{proof} We have
$$
G_m := (1-\la_m)G_{m-1}+\la_m\varphi_m = G_{m-1}+\la_m(\varphi_m-G_{m-1})
$$
and
$$
E(G_m) = \inf_{0\le \la\le 1}E(G_{m-1}+\la(\varphi_m-G_{m-1})).
$$
As for (\ref{2.6}) we have for any $\la$
$$
E(G_{m-1}+\la (\varphi_m-G_{m-1}))
$$
\begin{equation}\label{3.1}
  \le E(G_{m-1}) - \la\<-E'(G_{m-1}),\varphi_m-G_{m-1}\> + 2 \rho(E,2\la)
\end{equation}
and by (1) from the definition of the WRGA(co) and Lemma \ref{L2.2} we get
$$
\<-E'(G_{m-1}),\varphi_m-G_{m-1}\> \ge t_m \sup_{g\in \D} \<-E'(G_{m-1}),g-G_{m-1}\> =
$$
$$
t_m\sup_{\phi\in A_1(\D)} \<-E'(G_{m-1}),\phi-G_{m-1}\> \ge t_m   \<-E'(G_{m-1}),f-G_{m-1}\>.
$$
By (\ref{1.3})   we obtain
$$
\<-E'(G_{m-1}),f-G_{m-1}\> \ge E(G_{m-1})-E(f).
$$
Thus,  
$$
E(G_m) \le \inf_{0\le\la\le1} E(G_{m-1}+ \la(\varphi_m-G_{m-1}))  
$$
\begin{equation}\label{3.2}
\le E(G_{m-1}) + \inf_{0\le\la\le1}(-\la t_m (E(G_{m-1})-E(f)) + 2\rho(E,2\la),  
\end{equation}
which proves the lemma.
\end{proof}

The remaining part of the proof uses the inequality (\ref{3.2}) in the same way relation (\ref{2.7}) was used in the proof of Theorems \ref{T2.1o} and \ref{T2.3o}. The only additional difficulty here is that we are optimizing over $0\le \la \le 1$.  
In the proof of Theorem \ref{T3.1} we choose $\theta =\a/8$, assuming that $\alpha$ is small enough to guarantee that $\theta\le \theta_0$ and $\la = \xi_m(\rho,\tau,\theta)/2$. 

We proceed to the proof of Theorem \ref{T3.2}.
Denote
$$
a_n:=E(G_n)-E(f).
$$
The sequence $\{a_n\}$ is non-increasing. If $a_n\le0$ for some $n\le m$ then 
$E(G_m)-E(f)\le 0$   which implies Theorem \ref{T3.2}. Thus we assume that $a_n>0$ for $n\le m$.
We obtain from Lemma \ref{L3.1}
$$
a_m\le a_{m-1} +\inf_{0\le \la\le 1}(-\la t_ma_{m-1}+2\gamma(2\la)^q).
$$
We choose $\la$ from the equation
\begin{equation}\label{3.3}
 \la t_m a_{m-1}= 4\gamma(2\la)^q 
\end{equation}
if it is not greater than $1$ and choose $\la=1$ otherwise. 
The sequence $\{a_k\}$ is monotone decreasing and therefore we may choose $\la=1$ only at first $n$ steps and then choose $\la$ from (\ref{3.3}).
Then we get for $k\le n$
$$
a_k\le a_{k-1}(1-t_k/2)
$$
and
\begin{equation}\label{3.4}
a_n\le a_{0}\prod_{k=1}^n(1-t_k/2).
\end{equation}
For $k>n$ we have
\begin{equation}\label{3.5}
a_k\le a_{k-1}(1-\la t_k/2),\quad \la= \left(\frac{t_ma_{m-1}}{2^{2+q}\gamma}\right)^{\frac{1}{q-1}}.
\end{equation}
As in the proof of Theorem \ref{T2.3o} we obtain using Lemma \ref{L2.4}
$$
\frac{1}{y_m}\ge \frac{1}{y_n} +\sum_{k=n+1}^m w_k,\quad y_k:=a_k^{\frac{1}{q-1}},\quad w_k:=\frac{t_k^p}{2(2^{2+q}\gamma)^{\frac{1}{q-1}}}.
$$
By (\ref{3.4}) we get
$$
\frac{1}{y_n} \ge \frac{1}{y_0}\prod_{k=1}^n(1-t_k/2)^{\frac{1}{1-q}}.
$$
Next,
$$
\prod_{k=1}^n(1-t_k/2)^{\frac{1}{1-q}}\ge \prod_{k=1}^n(1+t_k/2)^{\frac{1}{q-1}} \ge \prod_{k=1}^n(1+ t_k/2)
$$
$$
\ge 1+\frac{1}{2}\sum_{k=1}^n t_k\ge 1+\frac{1}{2}\sum_{k=1}^n t_k^p.
$$
Combining the above inequalities we complete the proof.
\end{proof}
 
   \section{ Free relaxation} 

Both of the above algorithms, the WCGA(co) and the WRGA(co), use the functional $E'(G_{m-1})$ in a search for the $m$th element $\varphi_m$ from the dictionary to be used in optimization. The construction of the approximant in the WRGA(co) is different from the construction in the WCGA(co). In the WCGA(co) we build the approximant $G_m$ so as to maximally use the minimization power of the elements $\varphi_1,\dots,\varphi_m$. The WRGA(co) by its definition is designed for working with functions from $A_1(\D)$. In building the approximant in the WRGA(co) we keep the property $G_m\in A_1(\D)$.   As we mentioned in Section 3 the relaxation parameter $\lambda_m$ in the WRGA(co) is chosen at the $m$th step depending on $E$. The following modification of the above idea of relaxation in greedy approximation  will be studied in this section (see \cite{T26}).

  {\bf Weak Greedy Algorithm with Free Relaxation  (WGAFR(co)).} 
Let $\tau:=\{t_m\}_{m=1}^\infty$, $t_m\in[0,1]$, be a weakness  sequence. We define   $G_0  := 0$. Then for each $m\ge 1$ we have the following inductive definition.

(1) $\varphi_m   \in \D$ is any element satisfying
$$
\<-E'(G_{m-1}),\varphi_m\> \ge t_m  \sup_{g\in\D}\<-E'(G_{m-1}),g\>.
$$

(2) Find $w_m$ and $ \lambda_m$ such that
$$
E((1-w_m)G_{m-1} + \la_m\varphi_m) = \inf_{ \la,w}E((1-w)G_{m-1} + \la\varphi_m)
$$
and define
$$
G_m:=   (1-w_m)G_{m-1} + \la_m\varphi_m.
$$

\begin{Remark}\label{R4.1} It follows from the definition of the WGAFR(co) that 
the sequence $\{E(G_m)\}$ is a non-icreasing sequence. 
\end{Remark}

We begin with an analog of Lemma \ref{L2.3}.
 \begin{Lemma}\label{L4.1} Let $E$ be a uniformly smooth convex function with modulus of smoothness $\rho(E,u)$. Take a number $\e\ge 0$ and an element   $f^\e$ from $D$ such that
$$
E(f^\e) \le \inf_{x\in D}E(x)+ \e,\quad
f^\e/A(\e) \in A_1(\D),
$$
with some number $A(\e)\ge1$.
Then we have for the WGAFR(co)
$$
E(G_m)-E(f^\e) \le E(G_{m-1})-E(f^\e) 
$$
$$
+ \inf_{\la\ge0}(-\la t_mA(\e)^{-1}(E(G_{m-1})-E(f^\e)) + 2\rho(E,C_0\la)),
$$
for $ m=1,2,\dots$ .
\end{Lemma}
\begin{proof} By the definition of $G_m$
$$
E(G_m)\le \inf_{\la\ge0,w}E(G_{m-1}-wG_{m-1}+\la\ff_m).
$$
As in the arguments in the proof of Lemma \ref{L2.3} we
use Lemma \ref{L1.1} 
$$
E(G_{m-1}+\la \varphi_m -wG_{m-1})  \le E(G_{m-1})
$$
\begin{equation}\label{4.1}
   - \la\<-E'(G_{m-1}),\varphi_m\> -w\<E'(G_{m-1}),G_{m-1}\> + 2 \rho(E,\|\la\ff_m -wG_{m-1}\|)
\end{equation}
and estimate  
$$
\<-E'(G_{m-1}),\varphi_m\> \ge t_m \sup_{g\in \D} \<-E'(G_{m-1}),g\> =
$$
$$
t_m\sup_{\phi\in A_1(\D)} \<-E'(G_{m-1}),\phi\> \ge t_m A(\e)^{-1} \<-E'(G_{m-1}),f^\e\>.
$$
  We set $w^*:=\la t_mA(\e)^{-1}$ and obtain
  $$
  E(G_{m-1}-w^*G_{m-1}+\la\ff_m)
  $$
\begin{equation}\label{4.2}
 \le E(G_{m-1})-\la t_mA(\e)^{-1}\<-E'(G_{m-1}),f^\e-G_{m-1}\>.
\end{equation}
By   (\ref{1.3}) we obtain
$$
  \<-E'(G_{m-1}),f^\e-G_{m-1}\> \ge E(G_{m-1})-E(f^\e).
$$
 Thus, 
 $$
 E(G_m)\le E(G_{m-1})
 $$ 
\begin{equation}\label{4.2'}
  + \inf_{\la\ge0}(-\la t_mA(\e)^{-1}(E(G_{m-1})-E(f^\e)) + 2\rho(E,\|\la\ff_m -w^*G_{m-1}\|).  
\end{equation}
We now estimate
$$
\|w^*G_{m-1}-\la\ff_m\| \le w^*\|G_{m-1}\|+\la.
$$
Next, $E(G_{m-1})\le E(0)$ and, therefore, $G_{m-1}\in D$. Our assumption on boundedness of $D$ implies that $\|G_{m-1}\|\le C_1$.
Thus, under assumption $A(\e)\ge1$ we get
$$
w^*\|G_{m-1}\|\le C_1\la t_m  \le C_1\la.
$$
Finally,
$$
\|w^*G_{m-1}-\la\ff_m\|\le C_0\la.
$$
This completes the proof of Lemma 4.1.
\end{proof}

We now prove a convergence theorem for an arbitrary uniformly smooth convex function. Modulus of smoothness $\rho(E,u)$ of a uniformly smooth convex function is an even convex function such that $\rho(E,0)=0$ and  
$$
\lim_{u\to0}\rho(E,u)/u=0.
$$
  
\begin{Theorem}\label{T4.1} Let $E$ be a uniformly smooth convex function with modulus of smoothness $\rho(E,u)$. Assume that a sequence $\tau :=\{t_k\}_{k=1}^\infty$ satisfies the following condition. For any $\theta \in(0,\theta_0]$ we have
\begin{equation}\label{4.3}
\sum_{m=1}^\infty t_m  \xi_m(\rho,\tau,\theta)=\infty.  
\end{equation}
 Then, for the WGAFR(co) we have
$$
\lim_{m\to \infty} E(G_m) =\inf_{x\in D}E(x).
$$
\end{Theorem} 
\begin{proof}  By Remark \ref{R4.1}, $\{E(G_m)\}$ is a non-increasing sequence. Therefore we have
$$
\lim_{m\to \infty}E(G_m) =a.
$$
Denote
$$
b:=\inf_{x\in D}E(x),\quad \a:= a-b.
$$
We prove that $\a =0$ by contradiction. Assume to the contrary that $\a>0$. Then, for any $m$ we have
$$
E(G_m)-b \ge \a.
$$
We set $\e =\a/2$ and find $f^\e$ such that
$$
E(f^\e) \le b+\e \quad \text{and}\quad f^\e/A(\e) \in A_1(\D),
$$
with some $A(\e)\ge 1$. Then, by Lemma \ref{L4.1} we get
$$
E(G_m)-E(f^\e)  \le E(G_{m-1}) -E(f^\e) + \inf_{\la\ge0} (-\la t_mA(\e)^{-1}\a/2    +2\rho(E,C_0\la)).
$$
 
Let us specify $\theta:=\min\left(\theta_0,\frac{\a}{8A(\e)}\right)$ and take $\la =  C_0\xi_m(\rho,\tau,\theta)$. Then we obtain
$$
E(G_m) \le E(G_{m-1}) -2\theta t_m\xi_m.
$$
The assumption
$$
\sum_{m=1}^\infty t_m\xi_m =\infty
$$
brings a contradiction, which proves the theorem.
\end{proof} 

\begin{Theorem}\label{T4.2} Let $E$ be a uniformly smooth convex function with modulus of smoothness $\rho(E,u)\le \gamma u^q$, $1<q\le 2$. Take a number $\e\ge 0$ and an element  $f^\e$ from $D$ such that
$$
E(f^\e) \le \inf_{x\in D}E(x)+ \e,\quad
f^\e/A(\e) \in A_1(\D),
$$
with some number $A(\e)\ge 1$.
Then we have ($p:=q/(q-1)$)
\begin{equation}\label{4.4}
E(G_m)-\inf_{x\in D}E(x) \le  \max\left(2\e, C(q,\gamma)A(\e)^q\left(C(E,q,\gamma)+\sum_{k=1}^mt_k^p\right)^{1-q}\right) . 
\end{equation}
\end{Theorem}
\begin{proof} Denote
$$
a_n:=E(G_n)-E(f^\e).
$$
By Lemma \ref{L4.1}   we have  
\begin{equation}\label{4.5}
a_m \le a_{m-1}+\inf_{\la\ge0} \left(-\frac{\la t_m a_{m-1}}{A(\e)} + 2\gamma (C_0\la)^q\right).
\end{equation}
Choose $\la$ from the equation
$$
\frac{\la t_m a_{m-1}}{A(\e)} = 4\gamma (C_0\la)^q.
$$
The rest of the proof repeats the argument from the proof of Theorem \ref{T2.3o}.
\end{proof}

\section{Comments}

We already mentioned in the Introduction that the technique used in this paper 
is a slight modification of the corresponding technique developed in approximation theory (see \cite{T15}, \cite{T2} and the book \cite{Tbook}). 
We now discuss this in more detail.   We pointed out in the Introduction that at a first glance the settings of 
approximation and optimization problems are very different. In the approximation problem an element $f\in X$ is given and our task is to find a sparse approximation of it. In optimization theory an energy function $E(x)$ is given and we should find an approximate sparse solution to the minimization problem. It turns out that the same technique can be used for solving both problems. In nonlinear approximation we use greedy algorithms, for instance WCGA and WGAFR, for solving this problem. The greedy step is the one where we look for $\varphi_m   \in \D$   satisfying
$$
F_{f_{m-1}}(\varphi_m  ) \ge t_m   \sup_{g\in\D}F_{f_{m-1}}(g).
$$
This step is based on the norming functional $F_{f_{m-1}}$. As we pointed out in the Introduction the norming functional $F_{f_{m-1}}$ is the derivative of the 
norm function $E(x):=\|x\|$. Clearly, we can reformulate our problem of approximation of $f$ as an optimization problem with $E(x):=\|f-x\|$. It is a convex function, however, it is not a uniformly smooth function in the sense of smoothness of convex functions. A way out of this problem is to consider 
$E(f,x,q):=\|f-x\|^q$ with appropriate $q$. For instance, it is known (see \cite{BGHV}) that if $\rho(u)\le \gamma u^q$, $1<q\le 2$, then $E(f,x,q)$ is a uniformly smooth convex function with modulus of smoothness of order $u^q$. 
Next, 
$$
E'(f,x,q)=-q\|f-x\|^{q-1}F_{f-x}.
$$
Therefore, the algorithms WCGA(co), WRGA(co) and WGAFR(co) coincide in this case with the corresponding algorithms WCGA, WRGA and WGAFR from 
approximation theory. In the proofs of approximation theory results we use 
inequality (\ref{1.1'}) and the trivial inequality
\begin{equation}\label{5.1}
\|x+uy\|\ge F_x(x+uy) = \|x\|+uF_x(y).
\end{equation}
In the proofs of optimization theory results we use Lemma \ref{L1.1} instead of 
inequality (\ref{1.1'}) and the convexity inequality (\ref{1.2}) instead of (\ref{5.1}). The rest of the proofs uses the same technique of solving the corresponding recurrent inequalities. 

Our smoothness assumption on $E$ was used in the proofs of all
theorems from Sections 2--4 in the form of Lemma \ref{L1.1}. This means that in all those theorems the assumption that $E$ has modulus of smoothness $\rho(E,u)$ can be replaced by the assumption that $E$ satisfies the inequality
\begin{equation}\label{5.2}
 E(x+uy)-E(x)-u\<E'(x),y\>\le 2\rho(E,u\|y\|), \quad x\in D.  
\end{equation}
Moreover, in Section 3, where we consider the WRGA(co), the approximants $G_m$ are forced to stay in the $A_1(\D)$. Therefore, in Theorems \ref{T3.1} and \ref{T3.2} we can use the following inequality instead of (\ref{5.2})
\begin{equation}\label{5.3}
 E(x+u(y-x))-E(x)-u\<E'(x),y-x\>\le 2\rho(E,u\|y-x\|),  
 \end{equation}
for $x,y\in A_1(\D)$ and $u\in[0,1]$.

We note that smoothness assumptions in the form of (\ref{5.3}) with $\rho(E,u\|y-x\|)$ replaced by $C\|y-x\|^q$ were used in \cite{TRD}. The authors studied the version of WRGA(co) with weakness sequence $t_k=1$, $k=1,2,\dots$. They proved  Theorem \ref{T3.2} in this case. Their proof alike our proof in Section 3 is very close to the corresponding proof from greedy approximation (see \cite{T15}, \cite{T2} Section 3.3 or \cite{Tbook} Section 6.3).

We now make some general remarks on the results of this paper. As we already pointed out in Introduction a typical 
problem of convex optimization is to find an approximate solution to the problem
\begin{equation}\label{5.4}
w:=\inf_x E(x).
\end{equation}
In this paper we are interested in sparse (with respect to a given dictionary $\D$) solutions of (\ref{5.4}). This means that we are solving the following problem instead of (\ref{5.4}). For a given dictionary $\D$ consider the set of all $m$-term polynomials with respect to $\D$:
$$
\Sigma_m(\D):=\{x\in X: x=\sum_{i=1}^m c_ig_i,\quad g_i\in\D\}.
$$
We solve the following {\it sparse optimization problem}
\begin{equation}\label{5.5}
w_m:=\inf_{x\in\Sigma_m(\D)} E(x).
\end{equation}
In this paper we have used greedy-type algorithms to solve (approximately) problem (\ref{5.5}). 
Results of the paper show that it turns out that greedy-type algorithms with respect to $\D$ solve
problem (\ref{5.4}) too. 

We are interested in a solution from $\Sigma_m(\D)$. Clearly, when we optimize a linear form $\<F,g\>$ over the dictionary $\D$ we obtain the same 
value as optimization over the convex hull $A_1(\D)$. We often use this property (see Lemma \ref{L2.2}). However, at the greedy step of our algorithms we choose  

(1) $\varphi_m :=\varphi^{c,\tau}_m \in \D$ is {\bf any} element satisfying
$$
\<-E'(G_{m-1}),\varphi_m\> \ge t_m  \sup_{g\in \D}\< -E'(G_{m-1}),g\>.
$$
Thus if we replace the dictionary $\D$ by its convex hull $A_1(\D)$ we may 
take an element satisfying the above greedy condition which is not from $\D$ and could be even an infinite combination of the dictionary elements. 

Next, we begin with a Banach space $X$ and a convex function $E(x)$ defined on this space. Properties of this function $E$ are formulated in terms of Banach space $X$. If instead of Banach space $X$ we consider another Banach space, for instance, the one generated by $A_1(\D)$ as a unit ball then the properties of $E$ will change. For instance, a typical example of $E$ could be $E(x):=\|f-x\|^q$ with $\|\cdot\|$ being the norm of Banach space $X$.
Then our assumption that the set
$
D:=\{x:E(x)\le E(0)\}
$
is bounded is satisfied. However, this set is not necessarily bounded in the norm generated by $A_1(\D)$.

{\bf Acknowledgements.} This paper was motivated by the IMA Annual Program Workshop "Machine Learning: Theory and Computation" (March 26--30, 2012), in particular, by talks of Steve Wright and Pradeep Ravikumar.
The author is very thankful to Arkadi Nemirovski for an interesting discussion of the results and for his remarks.

\end{document}